\documentclass[twocolumn]{IEEEtran}
\usepackage{xcolor}

\usepackage{amsmath}
\usepackage{amssymb}
\usepackage{amsthm}
\usepackage{graphicx}
\usepackage{caption}
\usepackage{algorithm}
\usepackage{algpseudocode}
\usepackage{hyperref}
\usepackage[affil-it]{authblk}
\usepackage{mdframed}
\usepackage{booktabs}

\usepackage[numbers, sort&compress]{natbib}

\renewcommand{\(}{\left(}
\renewcommand{\)}{\right)}

\newcommand{\n}{\boldsymbol{n}}

\newcommand{\Sig}{\boldsymbol{\Sigma}}

\newcommand{\y}{\boldsymbol{y}}
\renewcommand{\H }{\boldsymbol{H}}

\newcommand{\R}{\mathbf{R}}

\newcommand{\1}{\mathbf{1}}

\newcommand{\x}{\boldsymbol{x}}

\newcommand{\C}{\mathbf{C}}

\newcommand{\A}{\boldsymbol{A}}

\newcommand{\Q}{\mathbf{Q}}

\newcommand{\q}{\boldsymbol{q}}

\newcommand{\EE}[1]{{\rm{E}}\left[#1\right]}

\newcommand{\norm}[1]{\left\|#1\right\|}

\renewcommand{\arg}[1]{{\rm{arg}}#1}

\newtheorem{definition}{Definition}
\newtheorem{theorem}{Theorem}

\begin{document}





\title{Learning to Estimate Without Bias}

\author[1]{Tzvi Diskin}
\author[2]{Yonina C. Eldar}%
\author[1]{Ami Wiesel}%
\affil[1]{The Hebrew University of Jerusalem}
\affil[2]{Weizmann Institute of Science}

\maketitle

\begin{abstract}
The Gauss–Markov theorem states that the weighted least squares estimator is a linear minimum variance unbiased estimation (MVUE) in linear models. In this paper, we take a first step towards extending this result to  non-linear settings via deep learning with bias constraints. The classical approach to designing non-linear MVUEs is through maximum likelihood estimation (MLE) which often involves computationally challenging optimizations. On the other hand, deep learning methods allow for non-linear estimators with fixed computational complexity. Learning based estimators perform optimally on average with respect to their training set but may suffer from significant bias in other parameters. To avoid this, we propose to add a simple bias constraint to the loss function, resulting in an estimator we refer to as Bias Constrained Estimator (BCE). We prove that this yields asymptotic MVUEs that behave similarly to the classical MLEs and asymptotically attain the Cramer Rao bound. We demonstrate the advantages of our approach in the context of signal to noise ratio estimation as well as covariance estimation. A second motivation to BCE is in applications where multiple estimates of the same unknown are averaged for improved performance. Examples include distributed sensor
networks and data augmentation in test-time. In such applications, we show that BCE leads to asymptotically consistent estimators.

\end{abstract}

\section{Introduction}

Parameter estimation is a fundamental problem in many areas of science and engineering. The goal is to recover an unknown deterministic parameter $\y$ given realizations of random variables $\x$ whose distribution depends on $\y$. An estimator $\hat{\y}(\x)$ is typically designed by minimizing some distance between the observed variables and their statistics, e.g., least squares, maximum likelihood  estimation (MLE) or method of moments \cite{kay1993fundamentals}. Performance is measured in terms of bias, variance and mean squared error (MSE), which depends on the unknowns. For example, MLE is known to be an asymptotic  Minimum Variance Unbiased Estimator (MVUE) for any value of $\y$.  In the last decade, there have been many works suggesting to design estimators using deep learning. Learning based estimators directly minimize the average MSE with respect to a given dataset. Their performance may deteriorate with respect to other values of $\y$. To close this gap, the goal of this paper is to introduce a machine learning framework for Bias Constrained Estimators (BCE).

The starting point to our work are the classical performance measures in parameter estimation. Statistics define the MSE and the bias of an estimate $\hat \y(\x)$ of $\y$ as (see for example Eq. 2.25 in \cite{friedman2017elements}, Chapter 2 in. \cite{kay1993fundamentals} or  \cite{lehmann1948completeness}):
\begin{align}\label{metrics_def}
{\rm{MSE}}_{\boldsymbol{\hat{y}}}(\boldsymbol{y})&=\rm{E}\left[\left\|\hat{\y}(\x)-\y\right\|^2\right]  \nonumber\\
{\rm{BIAS}}_{\boldsymbol{\hat{y}}}(\boldsymbol{y})&=\rm{E}\left[\hat{\y}(\x)-\y\right].
\end{align}
The ``M'' in the MSE and in all the expectations in (\ref{metrics_def}) are with respect to to the distribution of $\x$ parameterized by $\y$. The unknowns are not integrated out and the metrics are functions of $\y$ \cite{eldar2008rethinking}. The goal of parameter estimation is to minimize the MSE for any value of $\y$. This problem is ill-defined as different estimators are better for different values of $\y$. Therefore, it is standard to focus on minimizing the MSE only among unbiased estimators, that have zero bias for any value of $\y$. Fisher information provides a lower bound on this MSE known as the Cramer Rao bound (CRB). An unbiased estimator which achives the CRB is called an MVUE. Remarkably, the popular MLE is often asymptotically MVUE. See \cite{kay1993fundamentals} for more details on this topic. 


A competing framework is Bayesian statistics in which $\y$ is modeled as a random vector with a known prior. The goal of Bayesian estimation is also to minimize the MSE. But the definition of MSE is slightly different as the expectation is taken also with respect $\y$ (see for example Chapter 10 in \cite{kay1993fundamentals} or Chapter 2.4 in \cite{friedman2017elements}). To avoid confusion, we refer to this metric as BMSE :
\begin{equation}\label{bmse}
    {\rm{BMSE}}_{\boldsymbol{\hat{y}}} = \EE{{\rm{MSE}}_{\boldsymbol{\hat{y}}}(\y)}.
\end{equation}
Unlike MSE, BMSE is  not a function of $\y$. BMSE has a well defined minima but it is optimal on average with respect to $\y$ and depends on the chosen prior. In practice, this prior is often fictitious and does not really represent any prior knowledge on $\y$. Indeed, it is common to use simple uniform or wide Gaussian priors. Partial solutions to this issue include non-informative priors as Jeffrey's prior which require the Fisher information (or CRB) \cite{kass1996selection} and minimax approaches that are often too pessimistic  \cite{eldar2005minimax}.

In recent years, there is a growing trend of solving classical parameter estimation problems using deep learning. In brief, these methods are approximations to Bayesian methods where the underlying distribution is represented by a training dataset and the optimal solutions are implemented via expressive deep neural networks. This approach is advantageous when it is easy to collect or generate data but we do not have access to the exact probabilistic models. It also leads to networks that have fixed computational complexity which are often preferable. Examples include many fields including image reconstruction \cite{ongie2020deep,dong2015image,shlezinger2023model}, phase retrieval \cite{naimipour2020unfolded}, magnetic resonance imaging \cite{schlemper2018stochastic}, frequency estimation \cite{izacard2019data,dreifuerst2021signalnet}, Direction of arrival \cite{merkofer2022deep,shmuel2023deep}, channel estimation \cite{fesl2023low} and robust regression \cite{diskin2017deep}. Recent works also consider the use of neural networks for efficiently computing the theoretical Fisher information and CRBs \cite{habi2023learning,duy2022fisher}.

A main challenge in developing deep learning methods for parameter estimation is the choice of prior for $\y$. By construction, the learned networks are not optimal for specific values of $\y$ but only on average with respect to their training set. The goal of this paper is to close this gap by introducing the learned Bias Constrained Estimator (BCE). It minimizes the BMSE while promoting unbiasedness for every value of $\y$. We prove that BCE converges to an MVUE for every value of $\y$ independently of the prior. Specifically, given a rich enough architecture and a sufficiently large number of samples, BCE is asymptotically unbiased and achieves the lowest possible MSE for any value of $\y$.


To gain more intuition into the BCE we consider as a special case the linear settings. Here, the classical Gauss Markov theorem states that the Weighted Least Squares (WLS) estimate is a linear MVUE for any unknown parameter. On the other hand, the linear minimal mean square error estimator (LMMSE) is optimal with respect to a specific Bayesian prior. Both of these can be interpreted as instances of the linear BCE (LBCE). The LBCE has a closed form solution which is a regularized LMMSE with a hyper-parameter $\lambda$ that reduces the dependency on the prior. LBCE with $\lambda=0$ reduces to the LMMSE. With a large $\lambda$, LBCE converges to the WLS. Generally, BCE provides a flexible bridge between these two extremes, and extends them to non-linear settings.

Numerical experiments support the theory and show that BCE leads to near-MVUEs for all $\y$. Estimators based on BMSE alone are better on average but can be worse for specific values of $\y$. We demonstrate this in two synthetic settings with a known probabilistic model: signal to noise ratio (SNR) estimation and structured covariance estimation. Next, we illustrate the advantages of BCE in a real world localization problem where we only have access to a training set with no model. The results clearly show that the predictions of BCE are less accurate than its competitors but are unbiased and centered around the ground truth as needed in many applications.


The main purpose of BCE is estimation of deterministic parameters, but we also present an additional use case. Here, the motivation is in the context of averaging estimators in test time. In this setting the goal is to learn a single network that will be applied to multiple inputs and then take their average as the final output. This is the case, for example, in a sensor network where multiple independent measurements of the same phenomena are available. Each sensor applies the network locally and sends its estimate to a fusion center. The global center then uses the average of the estimates as the final estimate \cite{li2009distributed}. Averaging in test-time has also become standard in image classification where the same network is applied to multiple crops at inference time \cite{krizhevsky2012imagenet}. In such settings, unbiasedness of the local estimates is a goal on its own, as it is a necessary condition for asymptotic consistency of the global estimate. BCE enforces this condition and improves accuracy of the global estimator even over the BMSE. We demonstrate this advantage over minimizing the MSE of each of the local estimators in the context of image classification on CIFAR10 dataset with test-time data augmentation.


For completeness, we note that BCE is closely related to the topics of ``fairness'' and ``out of distribution (OOD) generalization'' which have recently attracted considerable attention in the  machine learning literature. The topics are related both in the terminology and in the solutions. 
Fair learning tries to eliminate biases in the training set and considers properties that need to be protected \cite{agarwal2019fair}. OOD works introduce an additional ``environment'' variable and the goal is to train a model that will generalize well on new unseen environments \cite{creager2021environment,maity2020there}. Among the proposed solutions are distributionally robust optimization \cite{bagnell2005robust} which is a type of minmax estimator, as well as invariant risk minimization \cite{arjovsky2019invariant} and calibration constraints \cite{wald2021calibration}, both of which are reminiscent of BCE. A main difference is that in our work the protected properties are the labels themselves. Another core difference is that in parameter estimation  we assume full knowledge of the generative model, whereas the above works are purely data-driven and discriminative.

The paper is organized as follows. In Section II, we formalize the problem and define the bias constrained estimator. Next, in Section III, we prove that it asymptotically converges to the MVUE. In Section IV, we analyze the linear case to get intuition into the way that BCE works. An additional application of averaging in test time is discussed in Section V. We implement BCE using deep neural networks and demonstrate its performance on different settings in section VI. Finally, we conclude and discuss some limitations of the work in Section VII.


\section{Biased Constrained Estimation}
\subsection{Classical Parameter Estimation}

Consider a random vector $\boldsymbol{x}$ whose probability distribution is parameterized by an unknown deterministic vector $\boldsymbol{y}$ in some region $S\subset\mathbb{R}^{D}$. We are interested in the case in which $\boldsymbol{y}$ is a deterministic variable without any prior distribution. We assume exact knowledge of the dependence of $\boldsymbol{x}$ on $\boldsymbol{y}$ via a likelihood function $p\left(\boldsymbol{x};\boldsymbol{y}\right)$. Typically, this knowledge is based on well specified parametric models, e.g. physics based models.  
Our goal is to estimate $\boldsymbol{y}$ given  $\boldsymbol{x}$.

The quality of an estimator $\hat \y(\x)$ is usually measured by the MSE (\ref{metrics_def}) which is a function of the unknown parameter $\y$. Minimizing the MSE for any value of $\y$ is an ill defined problem. A classical approach to bypass this issue is to consider only unbiased estimators as defined below. 
\begin{definition}
An estimator $\boldsymbol{\hat{y}}\left(\boldsymbol{x}\right)$ is called unbiased if it satisfies BIAS$_{\boldsymbol{\hat{y}}}(\boldsymbol{y})=\boldsymbol{0}$ for all $\boldsymbol{y}\in S$. 
\end{definition}
Among the unbiased estimators, the MVUE, as defined below, is optimal:
\begin{definition}
An MVUE is an unbiased estimator $\Tilde{\y}\left(\boldsymbol{x}\right)$ that has a variance lower than or equal to that of any other unbiased estimator for all values of $\y\in S$. 
\end{definition}

An MVUE does not always exists but can be guaranteed under favourable conditions, e.g.,  simple models in the exponential family  \cite[p. 88]{lehmann2006theory}.

In practice, the most common approach to parameter estimation is MLE which is asymptotically near MVUE. It is the solution to the following optimization
\begin{equation}
    \hat{\y}_{\rm{MLE}}=\arg\max_{\y} p(\x;\y).
\end{equation}
Under favourable conditions, MLE is an MVUE. First, this holds when the observation vector $\x=[x_1,\cdots,x_Q]^T$ consists of independent and identically distributed (i.i.d.) elements all conditioned on the same unknown $\y$, and  $q\Q\rightarrow \infty$ along with additional regularity conditions  \cite[p. 164]{kay1993fundamentals}. Second, in specific signal in noise problems, MLE is an MVUE even for short data records if the signal to noise ratio (SNR) is high enough. For more details, see Example 7.6 and Problem 7.15 in \cite{kay1993fundamentals}. Another advantage in these settings is that the performance of MLE attains the CRB: \cite{kay1993fundamentals}:
\begin{align}
    Q\cdot{\rm{cov}}_{\hat{\y}_{\rm{MLE}}}(\y)\overset{Q\rightarrow\infty}{\rightarrow} {\boldsymbol{F}}^{-1}(\y)
\end{align}
where ${\boldsymbol{F}}(\y)$ is the Fisher Information Matrix (FIM)
\begin{equation}
    \boldsymbol{F} = \EE{\left(\frac{\partial{\rm log} p(\x;\y)}{\partial\y}\right)^2},
\end{equation}
where $p(\x; \y)$ is the likelihood of a single sample.

 A main drawback is that it can be computationally expensive. In inference time, MLE requires the solution of a possibly non-convex and non-linear optimization for each observation vector $\x$. In many problems this optimization is intractable or impractical to implement.

\subsection{Deep Learning for Estimation}

The recent deep learning revolution has led many to apply machine learning methods for estimation \cite{dong2015image, ongie2020deep, gabrielli2017introducing, rudi2020parameter, dua2011artificial,dreifuerst2021signalnet}. Deep learning relies on a computationally intensive fitting phase which is done offline, and yields a neural network with fixed complexity that can be easily applied in inference time. Therefore, it is a promising approach for deriving low complexity estimators when MLE is intractable. To avoid confusion, we note that the main motivation to machine learning is usually in data-driven settings, while our main focus is on model-driven settings where deep learning is used to provide low complexity approximations to MLEs and MVUEs.

Deep learning relies on a training dataset, and therefore the first step in using it for parameter estimation is synthetic generation of a dataset 
\begin{equation}
    {\mathcal{D}}_N=\{\boldsymbol{y}_i,\boldsymbol{x}_i\}_{i=1}^N.
\end{equation}
For this purpose, the deterministic  $\boldsymbol{y}$ is assumed random with some fictitious prior $p^{\rm{fake}}(\boldsymbol{y})$ such that  $p^{\rm{fake}}(\y)\neq0$ for all $\y\in S$.  The dataset is then artificially generated according to $p^{\rm{fake}}(\boldsymbol{y})$ and the true $p\left(\boldsymbol{x};\boldsymbol{y}\right)$.
Next, a class of possible estimators ${\mathcal{H}}$ is chosen in order to tradeoff expressive power with computational complexity in test time. In the context of deep learning, the class ${\mathcal{H}}$ is usually a fixed differentiable neural network architecture. Finally, the learned estimator is defined as the minimizer of the empirical MSE (EMMSE):
\begin{equation}\label{EMMSE}
    {\rm{EMMSE}}:\;
   \min_{\hat{\boldsymbol{y}}\in {\mathcal{H}}} \frac 1N \sum_{i=1}^N \left\|\hat{\boldsymbol{y}}(\boldsymbol{x}_i)-\boldsymbol{y}_i \right\|^2.
\end{equation}

Assuming that $\y_i$ are i.i.d. and $N\rightarrow \infty$, the objective of EMMSE converges to the BMSE in (\ref{bmse}), where the expectation over $\y$ is with a respect to the fictitious prior. If ${\mathcal{H}}$ is sufficiently expressive, EMMSE converges to the Bayesian MMSE estimator which can be expressed as $\hat\y(\x)=\EE{\y|\x}$, \cite[p. 346]{kay1993fundamentals}. 

The performance of EMMSE estimators is usually promising but depends on the fictitious prior chosen for $\y$. In some cases, simple fake priors such as uniform or wide Gaussians are sufficient. But, in some settings, the MSE for specific values of $\y$ can be high and unpredictable. See for example the discussion on data generation in \cite{dreifuerst2021signalnet} where a special purpose generation mechanism was developed to avoid such failures. 

\subsection{BCE}
The contribution of this paper is a competing BCE approach that tries to gain the benefits of both worlds, namely learn a low-complexity deep learning estimator which is near MVUE and is less sensitive to the fictitious prior. For this purpose,  we minimize the empirical MSE along with an empirical squared bias regularization. We generate an enhanced dataset 
\begin{eqnarray}
{\mathcal{D}}_{NM}=\left\{\boldsymbol{y}_i,\{\boldsymbol{x}_{ij}\}_{j=1}^M\right\}_{i=1}^N
\end{eqnarray}
where the measurements $\boldsymbol{x}_{ij}$ are all generated  with the same $\boldsymbol{y}_i$. That is, we first generate $N$ samples $\{\y_i\}_{i=1}^N$ using the fictions prior and then we generate $M$ samples $\{x_{ij}\}_{j=1}^M$ for each of $y_i$, using $p(\x;\y_i\}$.
BCE is then defined as the solution to:
\begin{align}\label{BCE}
    {\rm{BCE}}:\;
    &\min_{\hat{\y}\in\mathcal{H}} \frac 1N \sum_{i=1}^N\norm{\hat{\y}(\x_i)-\y_i}^2 \nonumber\\
    &\qquad\qquad+ \lambda\frac 1N\sum_{i=1}^N{\norm{\frac 1M\sum_{j=1}^M{\hat{\y}(\x_{ij})-\y_i}}^2}.
\end{align}
where the second term is an empirical squared bias regularization and $\lambda\geq 0$ is a hyperparameter. By choosing $\mathcal{H}$ to be a differentiable neural network architecture as described above, optimizing (\ref{BCE}) can be done using standard deep learning tools such as stochastic gradient descent (SGD) \cite{shalev2014understanding} and its variants. The only difference from EMMSE is the enhanced dataset and the BCE regularization. The complete method is provided in Algorithm 1 below. 


\begin{algorithm}
\caption{BCE with synthetic data}\label{alg:BCE}
\begin{itemize}
    \item Choose a fictitious prior  $p^{\rm{fake}}(\y)$ .
    \item Generate $N$ samples $\{\y_i\}_{i=1}^N\;\sim\;p^{\rm{fake}}(\y)$.
    \item For each $\y_i,\;i=1,\cdots,N$:
    
    \hspace{4mm}Generate $M$ samples $\{\x_j(\y_i)\}_{j=1}^M\;\sim\;p(\x;\y_i)$.
    \item Solve the BCE optimization in (\ref{BCE}).
\end{itemize}
\end{algorithm}

\section{Asymptotic MVUE analysis}
In this section, we show that, under asymptotic conditions and a sufficiently expressive architecture, BCE converges to an MVUE and behaves like MLE. 

\begin{theorem}
 Assume that
 \begin{enumerate}
     \item An MVUE denoted by $\Tilde{\y}$ exists within ${\mathcal{H}}$. 
     \item The fictitious prior is non-singular, i.e., $p^{\rm{fake}}(\boldsymbol{y})\neq 0$ for all  ${\boldsymbol{y}}\in S$. 
     \item The variance of the MVUE estimate under the fictitious prior is finite: $\int {\rm{VAR}}_{\Tilde{\y}}(\boldsymbol{y}) p^{\rm{fake}}(\boldsymbol{y}) d\boldsymbol{y} < \infty$.
 \end{enumerate}
 Then, BCE converges to MVUE for sufficiently large $\lambda$, $M$ and $N$.
\end{theorem}

If an MVUE exists in the problem, the first assumption can be met by choosing a sufficiently expressive class of estimators. The second and third assumptions are technical and are less important in practice. 

\begin{proof}
The main idea of the proof is that because the squared bias is non negative, it is equal to zero for all $\y$ if and only if its expectation over any non-singular prior is zero. Thus, taking $\lambda$ to infinity enforces a solution that is unbiased for any value of $\y$. Among the unbiased solutions, only the MSE term in the BCE is left and thus the solution is the MVUE. 

For sufficiently large $M$ and $N$, the objective of (\ref{BCE}) converges to its population form: 
\begin{equation}\label{BCE_loss}
    L_{\rm{BCE}} = \EE{\norm{\hat{\y}(\x)-\y}^2} + \lambda \EE{\norm{\EE{\hat{\y}(\x)-\y|\y}}^2}.
\end{equation}
The MVUE satisfies
\begin{equation*}
    \EE{\Tilde{\y}(\x)-\y|\y}=0,\qquad \forall \; \y\in S.
\end{equation*}
Thus 
\begin{equation*}
    \EE{\norm{\EE{\Tilde{\y}(\x)-\y|\y}}^2}=0.
\end{equation*}

Now we show that the BCE objective (\ref{BCE_loss}) of the MVUE $\Tilde{\y}(\x)$ is smaller both from any biased solution $\hat{\y}_1(\x)$ and any unbiased solution $\hat{\y}_2(\x)$. 
First, assume that the BCE
 $\hat{\y}_1(\x)$ is biased for some $\y\in S$:
\begin{equation*}
    \EE{\hat{\y}_1(\x)-\y|\y}\neq0.
\end{equation*}
 Thus  
 \begin{equation*}
     \EE{\norm{\EE{\hat{\y}_1(\x)-\y|\y}}^2}>0.
 \end{equation*}
Since $\EE{\norm{\Tilde{\y}(\x)-\y}^2}$ is finite, this means that for sufficiently large $\lambda$ we have a contradiction
 \begin{equation*}
L_{\rm{BCE}}(\Tilde{\y}) < L_{\rm{BCE}}(\hat{\y}_1).
 \end{equation*}

Secondly, assume that the BCE $\hat{\y}_2$ is an unbiased estimator.
By the definition of MVUE, for all $\y\in S$ 
\begin{equation*}
    \EE{\norm{\Tilde{\y}(\x)-\EE{\Tilde{\y}|\y}}^2|\y}\leq\EE{\norm{\hat{\y}_2(\x)-\EE{\hat{\y}_2|\y}}^2|\y}
\end{equation*}
and
\begin{equation*}
    \EE{\norm{\Tilde{\y}(\x)-\y}^2|\y}\leq\EE{\norm{\hat{\y}_2(\x)-\y}^2|\y}.
\end{equation*}
Thus
\begin{eqnarray*}
    L_{\rm{BCE}}(\Tilde{\y})&=&\EE{\norm{\Tilde{\y}(\x)-\y}^2} \\ 
    &\leq & \EE{\norm{\hat{\y}_2(\x)-\y}^2} = L_{\rm{BCE}}(\hat{\y}_2).
\end{eqnarray*}
Together, the BCE loss of the MVUE is smaller than the BCE loss of any other estimator whether it is biased or not, completing the proof.
\end{proof}

To summarize, EMMSE (\ref{EMMSE}) is optimal on average with respect to the fictitious prior. Otherwise, it can perform badly (see experiments in Section \ref{experiments}).  On the other hand, if an MVUE for the problem exists, BCE is optimal for any value of $\y$ among the unbiased estimators, and is independent on the choice of the fictitious prior. In problems where a statistically efficient estimator exists, BCE achieves the Cramer Rao Bound (CRB), making its performance both optimal and predictable for any value of $\y$.

\section{Linear BCE}
In this section we focus on the linear case in which the BCE has a closed form. The section is divided into two parts. In the first part we restrict the estimator to be linear where in the second part we assume that  the likelihood model is also linear.

First we restrict attention to the class of linear estimators which is easy to fit, store and implement.
Note that using non linear features, this is applicable to any likelihood model and can be very expressive using random features \cite{rahimi2007random} or deep features (by replacing the last layer of a DNN, see for example in \cite{rosenfeld2022domain,kirichenko2022last}). For simplicity, we also assume perfect training with $N,M\rightarrow \infty$ so that the empirical means converge to their true expectations. 
The following theorem gives a closed form solution to the linear BCE (LBCE):
\begin{theorem}
Consider the the BCE optimization problem \ref{BCE} such that the estimator $\hat \y$ is restricted to the linear hypotheses class $\hat \y = \A\x$
Then the solution is given by:
\begin{align}
\label{lbce1}\A  &=  
\EE{\y\x^T}\left(\frac{1}{\lambda+1}\EE{\x\x^T}+\frac{\lambda}{\lambda+1}\R\right)^{-1}
\end{align}
where
\begin{align}
\R&=\EE{\EE{\x|\y}\EE{\x^T|\y}},
\end{align}
and we assume all are invertible. All the expectations over $\y$ are with respect to the fictitious prior.
\end{theorem}
The proof is straight forward by plugging the linear function into (\ref{BCE}), and taking the derivative to zero. The full derivation is in Appendix A.

Plugging in $\lambda=0$ in (\ref{lbce1}) yields the well known LMMSE  \cite[p. 380]{kay1993fundamentals}:
\begin{equation}\label{linreg}
\A \stackrel{\lambda=0}{\rightarrow} \EE{\y\x^T}\left(\EE{\x\x^T}\right)^{-1}.
\end{equation}
Just like LMMSE, the BCE in (\ref{lbce1}) holds for arbitrary distributions. It can account for highly non-linear relations between $\y$ and $\x$ by using pre-defined non-linear features of $\x$.

Next, we assume that the statistical relation between $\x$ and $\y$ is also linear plus additive zero mean noise. The solution to this setting is well known as the Gauss Markov theorem. As expected, BCE agrees with this special case.

\begin{theorem}
Consider the linear case where

\begin{align}\label{double_linear model}
    \boldsymbol{x}&=\boldsymbol{Hy+n}\nonumber\\
\hat{\boldsymbol{y}}&=\boldsymbol{A}{\boldsymbol{x}}
\end{align}
where ${\boldsymbol{n}}$ is a zero mean random vector with covariance $\Sig_{\n}$ and $\boldsymbol{y}$ is a parameter vector with a fictions prior with zero mean and covariance $\Sig_{\y}$.
For $N,M\rightarrow\infty$, LBCE reduces to 
\begin{align}\label{lbce2}
\A  &=   \left(\H^T\Sig_{\n}^{-1}\H+\frac{1}{\lambda+1}\Sig_{\y}^{-1}\right)^{-1}\H^T\Sig_{\n}^{-1} 
\end{align}
and we assume all are invertible. All the expectations over $\y$ are with respect to the fictitious prior.
\end{theorem}
The proof is simple and is based on plugging (\ref{double_linear model}) into (\ref{lbce1}) and using the matrix inversion lemma.

Equivalently, (\ref{lbce2}) become:
\begin{equation}
\label{bayeslinreg}
   \A=\left(\H^T\Sig_{\n}^{-1}\H+\Sig_{\y}^{-1}\right)^{-1}\H^T\Sig_{\n}^{-1}
\end{equation}
which is known as Bayesian linear regression \cite[p. 153]{bishop2006pattern}.
On the other hand, Plugging $\lambda\rightarrow\infty$ in (\ref{lbce2}) yields the seminal weighted least squares (WLS) estimator \cite[p. 226]{kay1993fundamentals}:
\begin{equation}\label{genWLS}
\A \stackrel{\lambda\rightarrow \infty}{\rightarrow} \left(\H^T\Sig_{\n}^{-1}\H\right)^{-1}\H^T\Sig_{\n}^{-1}.
\end{equation}
For finite $\lambda>0$, linear BCE provides a flexible bridge between the LMMSE and WLS estimators. The parameter $\lambda$ attenuates the effect of the fictitious prior by dividing the prior covariance in (\ref{lbce2}) by $(\lambda+1)$.

\section{BCE for Averaging}
The main motivation for BCE is approximating the MVUE using deep neural networks. 
In this section, we consider a different motivation to BCE where unbiasedness is a goal on its own. Specifically, BCEs are advantageous in scenarios where multiple estimators are combined together for improved performance. Typical examples include sensor networks where each sensor provides a local estimate of the unknown, or computer vision applications where the same network is applied on different crops of a given image to improve accuracy \cite{krizhevsky2012imagenet, szegedy2015going}. 

The underlying assumption in BCE for averaging is that we have access, both in training and in test times, to multiple $\x_{ij}$ associate with the same $\y_i$. This assumption holds in learning with synthetic data (e.g., Algorithm 1), or real world data with multiple views or augmentations. In any case, the data structure allows us to learn a single network $\hat{\y} (\cdot)$ which will be applied to each of them and then output their average as summarized in Algorithm 2. The following theorem then proves that BCE results in asymptotic consistency.

\begin{algorithm}
\caption{BCE for averaging}\label{alg:BCE2}
\begin{itemize}
    \item Let $\hat{\y}(\cdot)$ be the solution to  BCE in (\ref{BCE}).
    \item Define the global estimator as 
    \begin{equation}\label{fusion}
    \overline{\y}(\x)=\frac{1}{M_t}\sum_{j=1}^{M_t}\hat{\y}(\x_j),
\end{equation}
    where $M_t$ is the number of local estimators at test time.

\end{itemize}
\end{algorithm}

\begin{figure*}[h]
\centering
\captionsetup{justification=centering}

\center{\includegraphics[width=1.0\textwidth]{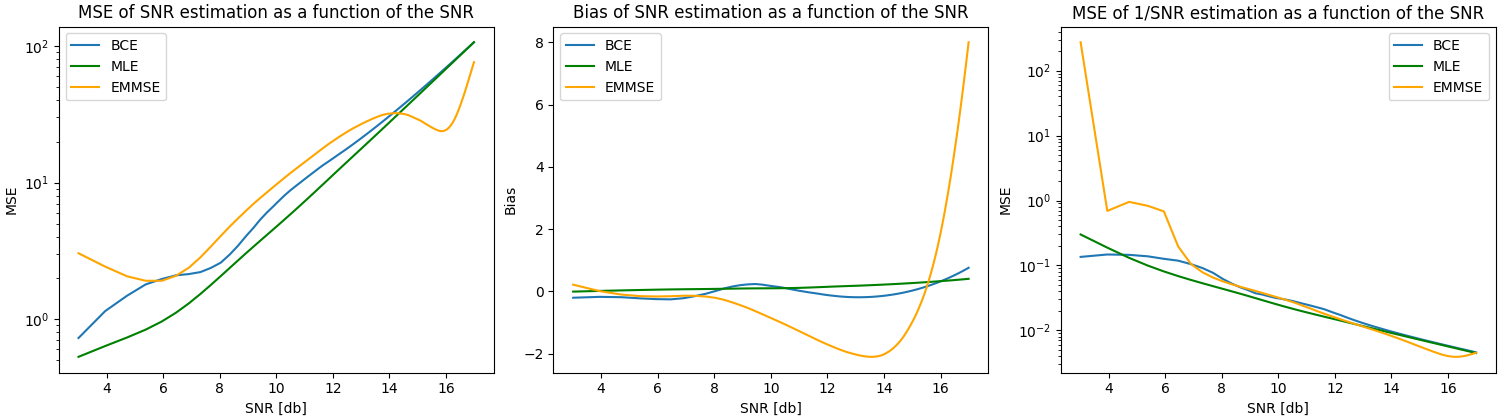}}
\caption{Bias of SNR, MSE of SNR and MSE of inverse SNR.}

\label{fig:bias_and_mse_SNR}%
\end{figure*}

\begin{theorem}
Assume that an unbiased estimator $\y^*(\x)$  with finite variance exists within the hypothesis class, and that a BCE was learned with 
sufficiently large $\lambda$, $M$ and $N$. Consider the case in which $M_t$ independent and identically distributed (i.i.d.) measurements $\x_j$ of the same $\y$ are available in test time. Then, the average BCE in (\ref{fusion}) is asymptotically consistent when $M_t$ increases. 
\end{theorem}

\begin{proof}
Following the proof of theorem 1, BCE with a sufficiently large $\lambda$, $M$ and $N$ results in a unbiased estimator, if one exists within the hypothesis class.
The global metrics satisfy:
\begin{align}
    {\rm{BIAS}}_{\overline{\y}}(\boldsymbol{y})&=\frac{1}{M}\sum_{j=1}^M{\rm{BIAS}}_{\hat{\y}}(\boldsymbol{y})={\rm{BIAS}}_{\hat{\y}}(\boldsymbol{y})\nonumber\\
     {\rm{VAR}}_{\overline{\y}}(\boldsymbol{y})&=\frac{1}{M^2}\sum_{j=1}^M{\rm{VAR}}_{\hat{\y}}(\boldsymbol{y})=\frac{1}{M}{\rm{VAR}}_{\hat{\y}}(\boldsymbol{y})\stackrel{M\rightarrow \infty}{\rightarrow}0.
\end{align}
Thus, the global variance decreases with $M_t$, whereas the global bias remains constant, and for an unbiased local estimator it is equal to zero.
\end{proof}

\section{Experiments}\label{experiments}

In this section, we present numerical experiments results. We focus on the main ideas and conclusions and leave the details to the appendix. Reproducible code with all the implementation details including architectures and hyperparameters will be provided at \url{https://github.com/tzvid/BCE}.

\subsection{SNR Estimation}\label{SNR Estimation}
Our first experiment addresses a non-convex estimation problem of a single unknown. The unknown is scalar and therefore we can easily compute the MLE and visualize its performance. Specifically, we consider 
non-data-aided SNR estimation \cite{alagha2001cramer}. The received signal is 
\begin{equation}
    x_l = a_lh + w_l
\end{equation}
where $a_l=\pm 1$ are equi-probable binary symbols, $h$ is an unknown signal and $w_l$ is a white Gaussian noise with unknown variance denoted by $\sigma^2$ and $l=1,\cdots,Q$. The goal is to estimate the SNR defined as $y=\frac{h^2}{\sigma^2}$. Different estimators were proposed for this problem, including MLE \cite{wiesel2002non} and method of moments \cite{pauluzzi2000comparison}. 
For our purposes, we compare the MLE with two identical networks trained on synthetic data using EMMSE and BCE loss functions.

Figure \ref{fig:bias_and_mse_SNR} shows the MSE and the bias of the different estimators as a function of the SNR. 
It is evident that BCE is a better approximation of MLE than EMMSE. EMMSE is very biased towards a narrow regime of the SNR. This is because the MSE scales as the square of the SNR and the average MSE loss is dominated by the large MSE examples. For completeness, we also plot the MSE in terms of inverse SNR defined as $\EE{\left(\frac{1}{\hat{y}}-\frac{1}{y}\right)^2}$.
Functional invariance is a well known and attractive property of MLE. The figure shows that both MLE and BCE are robust to the inverse transformation, whereas EMMSE is unstable and performs poorly in low SNR.

\subsection{Structured Covariance Estimation}\label{Structured Covariance Estimation}
Our second experiment considers a more interesting multivariate structured covariance estimation.  Specifically, we consider the estimation of a sparse covariance matrix \cite{chaudhuri2007estimation}. The measurement model is 
\begin{equation}
    p(\x;\Sig)\sim \mathcal{N}({\boldsymbol{0}},\Sig)\nonumber
\end{equation}
where $\y=[y^{(1)},\cdots,y^{(9)}]^T$ and
\begin{equation}\label{structured covariance}
\Sig = 
    \begin{pmatrix}
1+y^{(1)} & 0 & 0 & \frac{1}{2}y^{(6)} & 0 \\
0 & 1+y^{(2)} & 0 & \frac{1}{2}y^{(7)} & 0 \\
0 & 0 & 1+y^{(3)} & 0 & \frac{1}{2}y^{(8)} \\ 
\frac{1}{2}y^{(6)} & \frac{1}{2}y^{(7)} & 0 & 1+y^{(4)} & \frac{1}{2}y^{(9)} \\
0 & 0 & \frac{1}{2}y^{(8)} & \frac{1}{2}y^{(9)} & 1+y^{(5)} 
\end{pmatrix}  
\end{equation}
and $0\leq y^{(k)} \leq1$ are unknown parameters. We train a neural network using $p^{\rm{fake}}(y^{(k)})\sim \mathcal{U}(0,1)$ using both EMMSE and BCE.  Computing the MLE is non-trivial in these settings and we compare the performance to the theoretical asymptotic variance defined by the CRB \cite{kay1993fundamentals}.
The CRB depends on the specific values of $\y$. Therefore we take random realizations and provide scatter plots ordered according to the CRB value. Finally, as another baseline, we also report the result of an additional model named NORM. Following \cite{deng2023uncertainty}, NORM is trained with MSE normalized by the CRB. The idea is that the loss of each sample will be normalized by its "difficulty", and the performance will not be govern by the difficult examples. 

Figure \ref{fig:Structured Covariance} presents the results of two simulations. In the first, we generate data according to the training distribution $\y\sim\mathcal{U}(0,1)$. As expected, EMMSE which was optimized for this distribution provides the best MSEs. In the second, we generate test data according to a different distribution $\y\sim\mathcal{U}(0,0.1)$. Here the performance of EMMSE significantly deteriorates. NORM performs better in this out of distribution setting but is still far from the CRB. In contrast, in both settings BCE is near MVUE and provides MSEs close to the CRB while ignoring the prior distribution used in training. 

Note that BCE outperforms the NORM baseline and is also significantly easier to implement. NORM requires exact knowledge of the underlying model and the corresponding CRB. The main advantage of BCE is that it can be easily computed in a data-driven manner without access to the underlying model. The next example illustrates such real world settings.

\begin{figure}[h]
\centering

  \center{\includegraphics[width=0.45\textwidth]{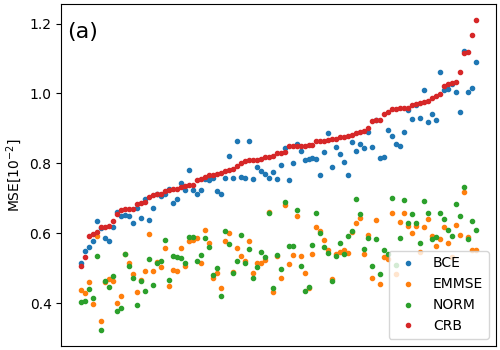}}
\center{\includegraphics[width=0.45\textwidth]{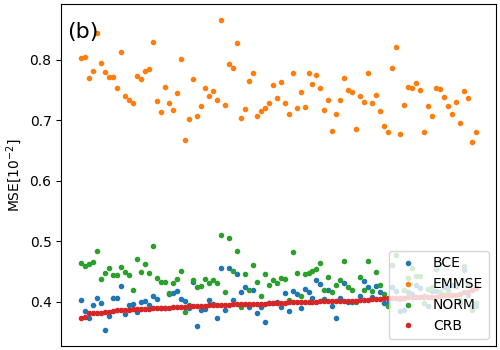}}
\caption{Scatter plot of MSEs ordered by CRB values. In (a) the tested $\y$'s are generated from the training distribution. In (b) the test distribution is different. BCE is near-MVUE and close to the CRB for both distributions, whereas EMMSE is better in (a) and weaker in (b).}

\label{fig:Structured Covariance}%
\end{figure}

\subsection{RSII Localization}\label{RSSI Localization}
Our third experiment considers BCE of unbiased localization using RSSI on real data. Specifically, we used the BLE RSSI Dataset for Indoor localization and Navigation Data Set from \cite{mohammadi2017semi}. The dataset contains 1420 labeled examples of RSSI readings of an array of 13 beacons as input and the xy location as outputs. We used two thirds of the data for training and a third for test. We evaluate the MSE and the average squared bias of EMMSE and BCE. To estimate the average squared bias, we use examples with the same true location. We only evaluate locations with at least 8 different examples. Table \ref{RSSI table} summarizes the results, showing that while EMMSE is better in terms of MSE, BCE results in a smaller bias. To gain more intuition, Fig \ref{fig:RSSI} illustrates this difference on a few selected examples. For each example, we draw the ground truth location (GT) and the model predictions using the different inputs. We also plot uncertainty ellipses that correspond to two standard deviations. It is easy to see that  while the EMMSE has smaller variance, the center of BCE estimations is closer to the true location.
\begin{table}[t]
\caption{MSE and Bias of EMSSE and BCE for RSSI Localization}
\vskip 0.15in
\begin{center}
\begin{sc}
\begin{tabular}{lcr}
\toprule
Metric/Model & MSE & Bias$^2$  \\
\midrule
EMMSE & 1.5 & 0.7 \\
BCE & 2.2 & 0.4 \\

\bottomrule
\end{tabular}
\end{sc}
\end{center}
\vskip -0.1in
\label{RSSI table}

\end{table}

\begin{figure}[h]
\centering

  \center{\includegraphics[width=0.25\textwidth]{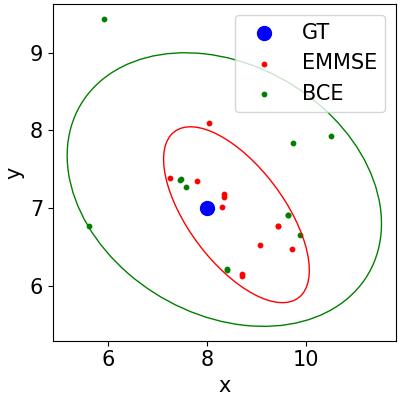}\includegraphics[width=0.25\textwidth]{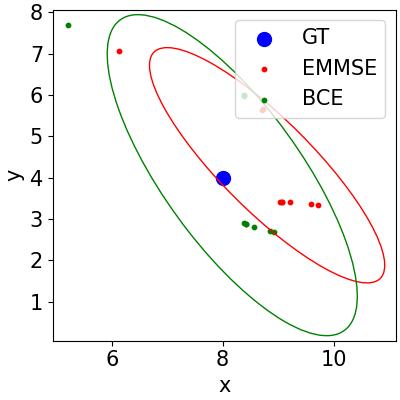}}
    \center{\includegraphics[width=0.25\textwidth]{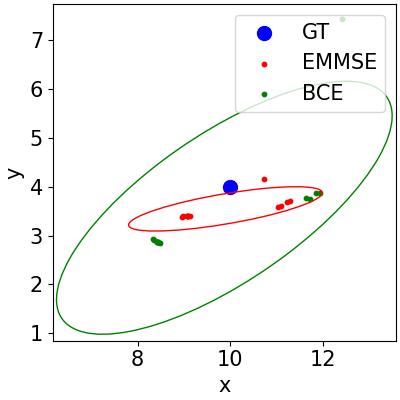}\includegraphics[width=0.25\textwidth]{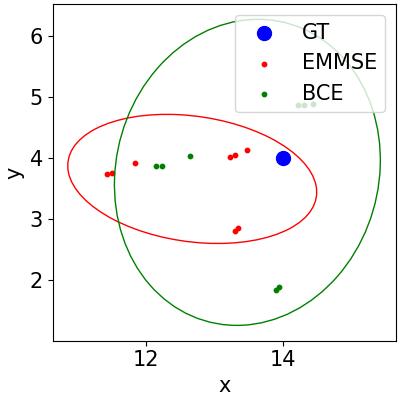}}
\caption{RSSI localization: Different estimates of EMMSE and BCE for the same location, for 4 different locations. While the estimations of EMMSE has smaller variance, the center of BCE estimations is closer to the true location.}

\label{fig:RSSI}%
\end{figure}

\subsection{Image Classification with Soft Labels }\label{Image Classification with Soft Labels}
Our fourth experiment considers BCE for averaging in the context of image classification. We consider the popular CIFAR10 dataset. BCE is designed for regression rather than classification. Therefore we consider soft labels as proposed in the knowledge distillation technique \cite{hinton2015distilling}. The soft labels are obtained using a strong ``teacher'' network from \cite{yu2018deep}. To exploit the benefits of averaging we rely on data augmentation in the training and test phases \cite{krizhevsky2012imagenet}. 
For augmentation, we use random cropping and flipping. We train two small ``student'' convolution networks with identical architectures using the EMMSE loss and the BCE loss. More precisely, following other distillation works, we also add a hard valued cross entropy term to the losses. 

Figure \ref{fig:CIFAR} compares the accuracy of EMMSE vs BCE as a function of the number of test-time data augmentation crops. It can be seen that while on a single crop EMMSE achieves a slightly better accuracy, BCE achieves better results when averaged over many crops in the test phase. 
\begin{figure}[h]
\centering

  \center{\includegraphics[width=0.45\textwidth]{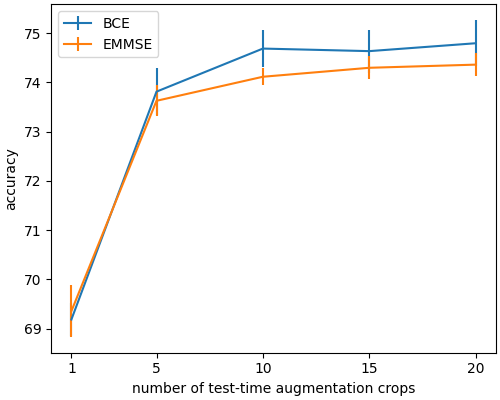}}
\caption{BCE for averaging: accuracy as a function of the number of test-time augmentation crops.}

\label{fig:CIFAR}%
\end{figure}

 \section{Summary}
In recent years, deep neural networks are replacing classical algorithms for estimation in many fields. While deep neural networks give remarkable improvement in performance ``on average'', in some situations one would prefer to use classical frequentist algorithms that have guarantees on their performance on any value of the unknown parameters. In this work we show that when a statistically efficient estimator exists, deep neural networks can be used to learn it using a bias constrained loss, provided that the architecture is expressive enough. 
 
BCE asymptotically converges to an MVUE but there are a few important caveats worth mentioning. First, in general problems, an MVUE does not always exist and the results only hold asymptotically. Just like MLE, BCE only attains the CRB when the number of samples is large and the errors are small. Otherwise, bias acts as a regularization mechanism which is typically preferable. Second, in many problems choosing a uniform or sufficiently wide prior leads to near MVUEs. BCE is mostly needed when the MSE changes significantly with the unknown parameters. In particular, two of our experiments involved variance estimation where the MSE scales quadratically with the unknown variance.



\section{ACKNOWLEDGMENT} This research was partially supported by ISF grant  2672/21.

\appendices

\section{Proofs of Theorems 2-3}
\subsection{Proof of theorem 2}
We insert $\hat{\y}=\A\x$ in (\ref{BCE}). Taking $M,N\rightarrow\infty$ gives the true expectations and thus:
\begin{align}
\rm{BCE}    &= {\rm{MSE}} + \lambda {\rm{BIAS}}^2 \nonumber\\
            &= \EE{\norm{\A\x-\y}^2} + \lambda \EE{\norm{\EE{\A\x-\y|\y}}^2}.
\end{align} 

Now the $\rm MSE$ term is equal to:

\begin{align}
    {\rm MSE} &= \EE{\norm{\A\x-\y}^2} \nonumber\\
              &= \EE{\x^T\A^T\A\x-\x^T\A^T\y-\y^T\A\x+\y^T\y} \nonumber\\
              &= {\rm Tr}\left(\A\EE{\x\x^T}\A^T - 2\A\EE{\x\y^T}+ \EE{\y\y^T}\right),
\end{align}
and the $\rm BIAS$ term is:
\begin{align}
    {\rm BIAS} &= \EE{\norm{\EE{\A\x-\y|\y}}^2} \nonumber\\
               &= \EE{\EE{\x^T\A^T|\y}\EE{\A\x|\y} - \EE{\x^T\A^T|\y}\y}  \nonumber\\
               &- \EE{\y^T\EE{\A\x|\y} + \y^T\y}\nonumber \\
               &= {\rm Tr}\left(\A\R\A^T- 2\A\EE{\EE{\x|\y}\y^T}\right) \nonumber\\
               &= {\rm Tr}\left(\A\R\A^T- 2\A\EE{\x\y^T}\right).
\end{align}
Thus:
\begin{align}
    {\rm BCE} &= {\rm Tr}\left(\A\left(\EE{\x\x^T}+\lambda\R\right)\A^T\right) \nonumber\\
              &- 2\left(\lambda+1\right){\rm Tr}\left(\A\EE{\x\y^T}\right) + \EE{\y^T\y}.
\end{align}
Taking the derivative with respect to $\boldsymbol{A}$ and
equating to zero:
\begin{align}
    \frac{\partial\rm{BCE}}{\partial\A} &=2\left(\EE{\x\x^T}+\lambda\R\right)\A^T 
                                        -2\left(\lambda+1\right)\EE{\x\y^T}.
\end{align}
Finally, 
\begin{align}
\A  &=  
\EE{\y\x^T}\left(\frac{1}{\lambda+1}\EE{\x\x^T}+\frac{\lambda}{\lambda+1}\R\right)^{-1}.
\end{align}

\subsection{Proof of theorem 3}
Using $\x=\H\y+ \n$, we obtain:
\begin{align}\label{ExyExx}
    \EE{\y\x^T} &= \EE{\y\y^T\H^T+\y\n^T} = \Sig_{\y}\H^T \nonumber \\
    \EE{\x\x^T} &= \EE{\H\y\y^T\H^T+\n\y\H^T+\H\y\n^T+\n\n^T} \nonumber\\
                &= \H\Sig_{\y}\H^T + \Sig_{n}, 
           \end{align}
and 
 \begin{equation}\label{R}
    \R          = \EE{\EE{\H\y+\n|\y}\EE{\H\y+\n|\y}^T} 
                = \H\Sig_{\y}\H^T.
 \end{equation}
Plugging (\ref{ExyExx}) and (\ref{R}) in (\ref{lbce1}) gives:
\begin{align}
        \A &= \Sig_{\y}\H^T\left(\H\Sig_{\y}\H^T+\frac{1}{\lambda+1}\Sig_{n}\right)^{-1} \nonumber\\
           &= \left(\H^T\Sig_{\n}^{-1}\H+\frac{1}{\lambda+1}\Sig_{\y}^{-1}\right)^{-1}\H^T\Sig_{\n}^{-1},
\end{align}
where the last step obtained using the matrix inversion lemma, completing the proof.

\section{Implementation Details}

\subsection{Implementation details for Section \ref{SNR Estimation}}
We train a simple fully connected model with one hidden layer. First the data is normalized by the second moment and then the input is augmented by hand crafted features: the fourth and sixth moments and different functions of them.
We train the network using $Q=50$ synthetic data in which the mean is sampled uniformly in $[1,10]$ and then SNR is sampled uniformly in $[2,50]$ (which corresponds to $[3dB,16dB]$). The data is generated independently in each batch. We trained the model using the standard MSE loss, and using BCE with $\lambda=1000$. We use ADAM solver with a multistep learning scheduler. We use batch sizes of $N=10$ and $M=100$ as defined (\ref{BCE}).

\subsection{Implementation details for Section \ref{Structured Covariance Estimation}}
We train a neural network for estimation the covariance matrix with the following architecture: First the sample covariance  $\C_0$ is calculated from the input $\x$ as it is the sufficient statistic for the covariance in Gaussian distribution. Also a vector $\boldsymbol{\alpha}_{k=0}$ is initialized to $\boldsymbol{\alpha}_{k=0}=\frac{1}{2}\1_9$ and a vector $\boldsymbol{v}_{k=0}$ is initialized to zero. Next, a one hidden layer fully connected network with concatenated input of $\C_0$ and $\boldsymbol{v}_{k=0}$ is used to predict a modification $\Delta\boldsymbol{\alpha}$ and $\Delta\boldsymbol{v}$ for the vectors $\boldsymbol{\alpha}_k$ and  $\boldsymbol{v}_{k}$ respectively, such that $\boldsymbol{\alpha}_{k+1}$ = $\boldsymbol{\alpha}_k + 0.1\Delta\boldsymbol{\alpha}$ and similarly for $\boldsymbol{v}_{k+1}$. Then an updated covariance $\C_{k+1}$ is calculated using $\boldsymbol{\alpha}_{k+1}$ and equation (\ref{structured covariance}). The process is repeated (with the updated $\C_k$ and $\boldsymbol{v}_k$ as an input to the fully connected network) for 50 iterations. The final covariance is the output of the network.
The network in is tranied on synthetic data in which the covariance the paramaters of the covariance are generated uniformly in their valid region and then $M$ different $X$'s are generated from a normal distribution with a zero mean and the generated covariance.
We use an ADAM solver with a "ReduceLROnPlateau" scheduler with the desired loss (BCE loss of BCE and MSE loss for EMMSE) on a synthetic validation set. 
We trained the model using the standard MSE loss, and using BCE with $\lambda=1000$.  We use batch sizes of $N=1$ and $M=20$ as defined in (\ref{BCE}). 

For the training of NORM, we use the loss:
\begin{align}
    \sum_{i,j}\(\hat \y(\x_{ij}\) - \y_i)^T \boldsymbol F(\y_i) \(\hat \y(\x_{ij}) - \y_i\)
\end{align}

\subsection{Implementation details for Section \ref{RSSI Localization}}
We use a fully connected neural network with a single hidden layer of size 20. The data (both the inputs and outputs) was normalized to have zero mean and unit variance at each dimension. For the bias term of the BCE loss, at each batch we uniformly sample a location and take the average of the output of the model for all the examples of the same location.

\subsection{Implementation details for Section \ref{Image Classification with Soft Labels}}

We  generate soft labels using a "teacher" network. Specifically, we work on the CIFAR10 dataset, and use a DLA architecture \cite{yu2018deep} which achieves 0.95 accuracy as a teacher. Our student network is a very small convolutional neural network (CNN) with two convolutions layers with 16 and 32 channels respectively and a single fully connected layer.
We now use the following notations: The original dataset is a set of $N$ triplets of images $\x_i$, a one-hot vector of hard labels $\y^h_i$  and a vector of soft labels $\y^s_i$ $\{\x_i,\y^h_i,\y^s_i\}$. In the augmented data, $M$ different images $\x_{ij}$ are generated randomly from each original image $\x_i$ using random cropping and flipping. The output of the network for the class $l$ is denoted by $z_{ij}^l$ and the vector of "probabilities" $\q_{ij}(T)$ is defined by: 
\begin{equation}
    q_{ij}^l(T)=\frac{\exp(q_{ij}^l/T)}{\sum_l\exp(z_{ij}^l/T)}
\end{equation}
where $T$ is a "temperature" that controls the softness of the probabilities.  
We define the following loss functions:
\begin{eqnarray}
L_{hard} &=& \sum_{ij}^{MN} CE(\q_{ij}(T=1), \y^h_i) \nonumber\\
L_{MSE} &=& \sum_{ij}^{MN} \norm{\q_{ij}(T=20) - \y^s_i} ^2 \nonumber\\
L_{bias} &=& \sum_i^N \norm{\sum_j^M \q_{ij}(T=20) - \y^s_i} ^2 
\end{eqnarray}
where $CE$ is the cross-entropy loss.
The regular network and the BCE are trained with the following losses:
\begin{eqnarray}
L_{EMMSE} &=& L_{hard}+L_{MSE} \nonumber\\
L_{BCE} &=& L_{hard}+L_{bias},
\end{eqnarray}
using stochastic gradient decent (SGD). In training, we use 20 different random crops and flips. We test the trained network in five different checkpoints during training and calculate the average and the standard deviation. In test-time we use the same data augmentation as in the training, the scores of the different crops are averaged to get the final score as in algorithm \ref{alg:BCE2}.
Fig. \ref{fig:CIFAR} shows the average and the standard deviation of the accuracy as a function of the number of crops in the training set.

\bibliographystyle{IEEEtranN}
\bibliography{main.bib}

\end{document}